\documentclass[a4,a4wide]{article}
\usepackage{aaai}
\usepackage{times}
\usepackage{helvet}
\usepackage{courier}
\usepackage{algorithm, algorithmic}

\usepackage{amsfonts}
\usepackage{amsmath,amssymb}
\usepackage{amsthm}
\usepackage{graphicx}
\usepackage[all,knot]{xy}
\theoremstyle{definition}

\newtheorem{definition}{Definition}

\newtheorem{theorem}{Theorem}

\newtheorem{story}{}

\newcommand{\notion}[1]{\textbf{\textit{#1}}}
\def\notaux{\notion}
\newcommand{\set}[1]{\left\{ #1 \right\}}
\newcommand{\tuple}[1]{\left\langle #1 \right\rangle}
\newcommand{\cond}[1]{\textit{(#1)}}

\newcommand{\argtype}[4][]{\mathtt{#2}#1(#3,#4)}
\renewcommand{\arg}[3][]{\argtype[_{#1}]{arg}{#2}{#3}}
\newcommand{\argsub}[3][]{\argtype[#1]{arg}{#2}{#3}}
\newcommand{\pro}[3][]{\argtype[_{#1}]{pro}{#2}{#3}}
\newcommand{\cau}[3][]{\argtype[_{#1}]{cau}{#2}{#3}}
\newcommand{\per}[3][]{\argtype[_{#1}]{per}{#2}{#3}}
\newcommand{\prc}[3][]{\argtype[_{#1}]{prec}{#2}{#3}}
\newcommand{\persimple}[2][]{\argtype[_{#1}]{per}{#2}{\cdot}}

\newcommand{\obs}[2]{\textsc{obs}(#1,#2)}

\def\narr{\mathcal{N}}
\def\wk{\mathcal{W}}
\def\story{\mathcal{SR}}

\def\f{\textsc{f}}
\def\b{\textsc{b}}

\def\G{\mathbb{G}}

\begin{document}

\nocopyright


\newcommand{\Affected}{\ensuremath{\mbox{\textit{Affected}}}}
\newcommand{\AntiTrajectory}{\ensuremath{\mbox{\textit{AntiTrajectory}}}}
\newcommand{\asort}{\ensuremath{{\cal A}}}
\newcommand{\Breaks}{\ensuremath{\mbox{\textit{Breaks}}}}
\newcommand{\BreaksTo}{\ensuremath{\mbox{\textit{BreaksTo}}}}
\newcommand{\Cancelled}{\ensuremath{\mbox{\textit{Cancelled}}}}
\newcommand{\Cancels}{\ensuremath{\mbox{\textit{Cancels}}}}
\newcommand{\CausesValue}{\ensuremath{\mbox{\textit{CausesValue}}}}
\newcommand{\CIRC}{\ensuremath{\mbox{\textit{CIRC}}}}
\newcommand{\Clipped}{\ensuremath{\mbox{\textit{Clipped}}}}
\newcommand{\Continuous}{\ensuremath{\mbox{\textit{Continuous}}}}
\newcommand{\Declipped}{\ensuremath{\mbox{\textit{Declipped}}}}
\newcommand{\defin}{\ensuremath{\stackrel{\mbox{{\tiny def}}}{\equiv}}}
\newcommand{\DifferAfterAtMostBy}{\ensuremath{\mbox{\textit{DifferAfterAtMostBy}}}}
\newcommand{\DifferAfterOnlyBy}{\ensuremath{\mbox{\textit{DifferAfterOnlyBy}}}}
\newcommand{\Differentiable}{\ensuremath{\mbox{\textit{Differentiable}}}}
\newcommand{\DifferInitiallyAtMostBy}{\ensuremath{\mbox{\textit{InitiallyDifferAtMostBy}}}}
\newcommand{\DifferInitiallyOnlyBy}{\ensuremath{\mbox{\textit{DifferAfterOnlyBy}}}}
\newcommand{\Dur}{\ensuremath{\mbox{\textit{Dur}}}}
\newcommand{\EqualUpTo}{\ensuremath{\mbox{\textit{EqualUpTo}}}}
\newcommand{\false}{\ensuremath{\mbox{\textit{false}}}}
\newcommand{\False}{\ensuremath{\mbox{\textit{False}}}}
\newcommand{\Frame}{\ensuremath{\mbox{\textit{Frame}}}}
\newcommand{\fsort}{\ensuremath{{\cal F}}}
\newcommand{\gequ}{\ensuremath{\widetilde{=}}}
\newcommand{\gneg}{\ensuremath{\widetilde{\neg}}}
\newcommand{\Goal}{\ensuremath{\mbox{\textit{Goal}}}}
\newcommand{\gsort}{\ensuremath{{\cal G}}}
\newcommand{\gvee}{\ensuremath{\widetilde{\vee}}}
\newcommand{\gwedge}{\ensuremath{\widetilde{\wedge}}}
\newcommand{\Happens}{\ensuremath{\mbox{\textit{Happens}}}}
\newcommand{\HoldsAt}{\ensuremath{\mbox{\textit{HoldsAt}}}}
\newcommand{\HoldsForm}{\ensuremath{\mbox{\textit{HoldsFormula}}}}
\newcommand{\HoldsIn}{\ensuremath{\mbox{\textit{HoldsIn}}}}
\newcommand{\IdenticalUpToOccurrencesOf}{\ensuremath{\mbox{\textit{IdenticalUpToOccurrencesOf}}}}
\newcommand{\IfFalse}{\ensuremath{\mbox{\textit{IfFalse}}}}
\newcommand{\IfTrue}{\ensuremath{\mbox{\textit{IfTrue}}}}
\newcommand{\Impossible}{\ensuremath{\mbox{\textit{Impossible}}}}
\newcommand{\InitiallyDifferAtMostBy}{\ensuremath{\mbox{\textit{InitiallyDifferAtMostBy}}}}
\newcommand{\InitiallyFalse}{\ensuremath{\mbox{\textit{InitiallyFalse}}}}
\newcommand{\InitiallyN}{\ensuremath{\mbox{\textit{InitiallyN}}}}
\newcommand{\InitiallyP}{\ensuremath{\mbox{\textit{InitiallyP}}}}
\newcommand{\InitiallyTrue}{\ensuremath{\mbox{\textit{InitiallyTrue}}}}
\newcommand{\Initiates}{\ensuremath{\mbox{\textit{Initiates}}}}
\newcommand{\isort}{\ensuremath{{\cal I}}}
\newcommand{\KnowsFalse}{\ensuremath{\mbox{\textit{KnowsFalse}}}}
\newcommand{\KnowsHappens}{\ensuremath{\mbox{\textnormal{\textit{KnowsHappens}}}}}
\newcommand{\KnowsNotHappens}{\ensuremath{\mbox{\textnormal{\textit{KnowsNotHappens}}}}}
\newcommand{\KnowsIfHappens}{\ensuremath{\mbox{\textnormal{\textit{KnowsIfHappens}}}}}
\newcommand{\KnowsHoldsForm}{\ensuremath{\mbox{\textit{KnowsHoldsFormula}}}}
\newcommand{\KnowsNotHoldsForm}{\ensuremath{\mbox{\textit{KnowsNotHoldsFormula}}}}
\newcommand{\KnowsTrue}{\ensuremath{\mbox{\textit{KnowsTrue}}}}
\newcommand{\KnowsValue}{\ensuremath{\mbox{\textnormal{\textit{KnowsValue}}}}}
\newcommand{\KnowsValueIs}{\ensuremath{\mbox{\textnormal{\textit{KnowsValueIs}}}}}
\newcommand{\KnowsValueIsNot}{\ensuremath{\mbox{\textnormal{\textit{KnowsValueIsNot}}}}}
\newcommand{\KnowsValueWasIsOrWillBe}{\ensuremath{\mbox{\textit{KnowsValueWasIsOrWillBe}}}}
\newcommand{\KnowsValueWasntIsntOrWontBe}{\ensuremath{\mbox{\textit{KnowsValueWasntIsntOrWontBe}}}}
\newcommand{\KnowsWhether}{\ensuremath{\mbox{\textit{KnowsWhether}}}}
\newcommand{\LeftContinuous}{\ensuremath{\mbox{\textit{LeftContinuous}}}}
\newcommand{\MacClipped}{\ensuremath{\mbox{\textbf{Clipped}}}}
\newcommand{\MacDeclipped}{\ensuremath{\mbox{\textbf{Declipped}}}}
\newcommand{\MacIdenticalUpToOccurrencesOf}{\ensuremath{\mbox{\textbf{IdenticalUpToOccurrencesOf}}}}
\newcommand{\MacKnowsFalse}{\ensuremath{\mbox{\textbf{KnowsFalse}}}}
\newcommand{\MacKnowsHappens}{\ensuremath{\mbox{\textbf{KnowsHappens}}}}
\newcommand{\MacKnowsHappensIfFalse}{\ensuremath{\mbox{\textbf{KnowsHappensIfFalse}}}}
\newcommand{\MacKnowsHappensIfTrue}{\ensuremath{\mbox{\textbf{KnowsHappensIfTrue}}}}
\newcommand{\MacKnowsTrue}{\ensuremath{\mbox{\textbf{KnowsTrue}}}}
\newcommand{\MacKnowsWhether}{\ensuremath{\mbox{\textbf{KnowsWhether}}}}
\newcommand{\Next}{\ensuremath{\mbox{\textit{Next}}}}
\newcommand{\Occurred}{\ensuremath{\mbox{\textit{Occurred}}}}
\newcommand{\Occurs}{\ensuremath{\mbox{\textit{Occurs}}}}
\newcommand{\Perform}{\ensuremath{\mbox{\textit{Perform}}}}
\newcommand{\PerformIfKnownFalse}{\ensuremath{\mbox{\textit{PerformIfKnownFalse}}}}
\newcommand{\PerformIfKnownTrue}{\ensuremath{\mbox{\textit{PerformIfKnownTrue}}}}
\newcommand{\PerformIfKnowsHoldsForm}{\ensuremath{\mbox{\textit{PerformIfKnowsHoldsFormula}}}}
\newcommand{\PerformIfKnowsNotHoldsForm}{\ensuremath{\mbox{\textit{PerformIfKnowsNotHoldsFormula}}}}
\newcommand{\PerformIfValueKnownIs}{\ensuremath{\mbox{\textit{PerformIfValueKnownIs}}}}
\newcommand{\PerformNonDet}{\ensuremath{\mbox{\textit{PerformNonDet}}}}
\newcommand{\PersistsBetween}{\ensuremath{\mbox{\textit{PersistsBetween}}}}
\newcommand{\Plan}{\ensuremath{\mbox{\textit{Plan}}}}
\newcommand{\Poss}{\ensuremath{\mbox{\textit{Poss}}}}
\newcommand{\PossVal}{\ensuremath{\mbox{\textit{PossVal}}}}
\newcommand{\ReleasedAt}{\ensuremath{\mbox{\textit{ReleasedAt}}}}
\newcommand{\ReleasedBetween}{\ensuremath{\mbox{\textit{ReleasedBetween}}}}
\newcommand{\Releases}{\ensuremath{\mbox{\textit{Releases}}}}
\newcommand{\RightLimit}{\ensuremath{\mbox{\textit{RightLimit}}}}
\newcommand{\rsort}{\ensuremath{\mathbb{R}_{\geq0}}}
\newcommand{\Sense}{\ensuremath{\mbox{\textit{Sense}}}}
\newcommand{\Senses}{\ensuremath{\mbox{\textit{Senses}}}}
\newcommand{\StartedIn}{\ensuremath{\mbox{\textit{StartedIn}}}}
\newcommand{\StoppedIn}{\ensuremath{\mbox{\textit{StoppedIn}}}}
\newcommand{\Terminates}{\ensuremath{\mbox{\textit{Terminates}}}}
\newcommand{\Theory}{\ensuremath{\mbox{\textit{Theory}}}}
\newcommand{\Trajectory}{\ensuremath{\mbox{\textit{Trajectory}}}}
\newcommand{\Triggered}{\ensuremath{\mbox{\textit{Triggered}}}}
\newcommand{\true}{\ensuremath{\mbox{\textit{true}}}}
\newcommand{\True}{\ensuremath{\mbox{\textit{True}}}}
\newcommand{\tsort}{\ensuremath{{\cal T}}}
\newcommand{\Value}{\ensuremath{\mbox{\textit{Value}}}}
\newcommand{\ValueAltered}{\ensuremath{\mbox{\textit{ValueAltered}}}}
\newcommand{\ValueAt}{\ensuremath{\mbox{\textit{ValueAt}}}}
\newcommand{\ValueCaused}{\ensuremath{\mbox{\textit{ValueCaused}}}}
\newcommand{\ValueOf}{\ensuremath{\mbox{\textit{ValueOf}}}}
\newcommand{\vsort}{\ensuremath{{\cal V}}}
\newcommand{\wsort}{\ensuremath{{\cal W}}}


\newcommand{\Alarm}{\ensuremath{\mbox{\textit{Alarm}}}}
\newcommand{\BlinkFunction}{\ensuremath{\mbox{\textit{BlinkFunction}}}}
\newcommand{\Blue}{\ensuremath{\mbox{\textit{Blue}}}}
\newcommand{\Bought}{\ensuremath{\mbox{\textit{Bought}}}}
\newcommand{\Broken}{\ensuremath{\mbox{\textit{Broken}}}}
\newcommand{\Cancel}{\ensuremath{\mbox{\textit{Cancel}}}}
\newcommand{\CloseValve}{\ensuremath{\mbox{\textit{CloseValve}}}}
\newcommand{\ColA}{\ensuremath{\mbox{\textit{Col}}_a}}
\newcommand{\Collected}{\ensuremath{\mbox{\textit{Collected}}}}
\newcommand{\CollectFromBlue}{\ensuremath{\mbox{\textit{CollectFromBlue}}}}
\newcommand{\CollectFromGreen}{\ensuremath{\mbox{\textit{CollectFromGreen}}}}
\newcommand{\CollectFromRed}{\ensuremath{\mbox{\textit{CollectFromRed}}}}
\newcommand{\CollectionPoint}{\ensuremath{\mbox{\textit{CollectionPoint}}}}
\newcommand{\ColM}{\ensuremath{\mbox{\textit{Col}}_m}}
\newcommand{\Countdown}{\ensuremath{\mbox{\textit{Countdown}}}}
\newcommand{\DieFaceShowing}{\ensuremath{\mbox{\textit{DieFaceShowing}}}}
\newcommand{\Display}{\ensuremath{\mbox{\textit{Display}}}}
\newcommand{\GateStatus}{\ensuremath{\mbox{\textit{GateStatus}}}}
\newcommand{\Flow}{\ensuremath{\mbox{\textit{Flow}}}}
\newcommand{\GiveFreeGift}{\ensuremath{\mbox{\textit{GiveFreeGift}}}}
\newcommand{\GoThrough}{\ensuremath{\mbox{\textit{GoThrough}}}}
\newcommand{\Green}{\ensuremath{\mbox{\textit{Green}}}}
\newcommand{\Happy}{\ensuremath{\mbox{\textit{Happy}}}}
\newcommand{\HasGrabber}{\ensuremath{\mbox{\textit{HasGrabber}}}}
\newcommand{\HasKey}{\ensuremath{\mbox{\textit{HasKey}}}}
\newcommand{\HeadsUp}{\ensuremath{\mbox{\textit{HeadsUp}}}}
\newcommand{\HeaterOn}{\ensuremath{\mbox{\textit{HeaterOn}}}}
\newcommand{\Height}{\ensuremath{\mbox{\textit{Height}}}}
\newcommand{\IndicatorDepressed}{\ensuremath{\mbox{\textit{IndicatorDepressed}}}}
\newcommand{\Insert}{\ensuremath{\mbox{\textit{Insert}}}}
\newcommand{\Inside}{\ensuremath{\mbox{\textit{Inside}}}}
\newcommand{\Level}{\ensuremath{\mbox{\textit{Level}}}}
\newcommand{\LiftLeft}{\ensuremath{\mbox{\textit{LiftLeft}}}}
\newcommand{\LiftRight}{\ensuremath{\mbox{\textit{LiftRight}}}}
\newcommand{\Light}{\ensuremath{\mbox{\textit{Light}}}}
\newcommand{\LightOn}{\ensuremath{\mbox{\textit{LightOn}}}}
\newcommand{\Location}{\ensuremath{\mbox{\textit{Location}}}}
\newcommand{\Locked}{\ensuremath{\mbox{\textit{Locked}}}}
\newcommand{\Lower}{\ensuremath{\mbox{\textit{Lower}}}}
\newcommand{\MotorEngaged}{\ensuremath{\mbox{\textit{MotorEngaged}}}}
\newcommand{\MoveEast}{\ensuremath{\mbox{\textit{MoveEast}}}}
\newcommand{\MovingEast}{\ensuremath{\mbox{\textit{MovingEast}}}}
\newcommand{\Off}{\ensuremath{\mbox{\textit{Off}}}}
\newcommand{\On}{\ensuremath{\mbox{\textit{On}}}}
\newcommand{\Open}{\ensuremath{\mbox{\textit{Open}}}}
\newcommand{\OpenValve}{\ensuremath{\mbox{\textit{OpenValve}}}}
\newcommand{\Pickup}{\ensuremath{\mbox{\textit{Pickup}}}}
\newcommand{\PressDoorBell}{\ensuremath{\mbox{\textit{PressDoorBell}}}}
\newcommand{\PressSwitch}{\ensuremath{\mbox{\textit{PressSwitch}}}}
\newcommand{\Pull}{\ensuremath{\mbox{\textit{Pull}}}}
\newcommand{\PurA}{\ensuremath{\mbox{\textit{Pur}}\!_a}}
\newcommand{\Purchase}{\ensuremath{\mbox{\textit{Purchase}}}}
\newcommand{\PurM}{\ensuremath{\mbox{\textit{Pur}}\!_m}}
\newcommand{\Push}{\ensuremath{\mbox{\textit{Push}}}}
\newcommand{\Raised}{\ensuremath{\mbox{\textit{Raised}}}}
\newcommand{\Red}{\ensuremath{\mbox{\textit{Red}}}}
\newcommand{\Removed}{\ensuremath{\mbox{\textit{Removed}}}}
\newcommand{\Ringing}{\ensuremath{\mbox{\textit{Ringing}}}}
\newcommand{\RingingNoise}{\ensuremath{\mbox{\textit{RingingNoise}}}}
\newcommand{\RollDie}{\ensuremath{\mbox{\textit{RollDie}}}}
\newcommand{\Set}{\ensuremath{\mbox{\textit{Set}}}}
\newcommand{\Shut}{\ensuremath{\mbox{\textit{Shut}}}}
\newcommand{\Smoke}{\ensuremath{\mbox{\textit{Smoke}}}}
\newcommand{\Spilt}{\ensuremath{\mbox{\textit{Spilt}}}}
\newcommand{\StartMoveEast}{\ensuremath{\mbox{\textit{StartMoveEast}}}}
\newcommand{\StartRing}{\ensuremath{\mbox{\textit{StartRing}}}}
\newcommand{\StopMoveEast}{\ensuremath{\mbox{\textit{StopMoveEast}}}}
\newcommand{\SwitchedOn}{\ensuremath{\mbox{\textit{SwitchedOn}}}}
\newcommand{\TailsUp}{\ensuremath{\mbox{\textit{TailsUp}}}}
\newcommand{\Tick}{\ensuremath{\mbox{\textit{Tick}}}}
\newcommand{\Tock}{\ensuremath{\mbox{\textit{Tock}}}}
\newcommand{\TossCoin}{\ensuremath{\mbox{\textit{TossCoin}}}}
\newcommand{\TossHead}{\ensuremath{\mbox{\textit{TossHead}}}}
\newcommand{\TossTail}{\ensuremath{\mbox{\textit{TossTail}}}}
\newcommand{\ValveOpen}{\ensuremath{\mbox{\textit{ValveOpen}}}}

\title{Non-Monotonic Reasoning and Story Comprehension\\{\small In Proceedings of the 15th International Workshop on Non-Monotonic Reasoning (NMR 2014), Vienna, 17--19 July, 2014}}
\author{Irene-Anna Diakidoy\\
University of Cyprus\\
eddiak@ucy.ac.cy
\And
Antonis Kakas\\
University of Cyprus\\
antonis@cs.ucy.ac.cy
\And
Loizos Michael\\
Open University of Cyprus\\
loizos@ouc.ac.cy
\And
Rob Miller\\
University College London\\
rsm@ucl.ac.uk
}

\maketitle

\begin{abstract}\label{abstract}
This paper develops a Reasoning about Actions and Change
framework integrated with Default Reasoning,
suitable as a Knowledge Representation and Reasoning framework
for Story Comprehension.
The proposed framework, which is guided strongly by existing knowhow
from the Psychology of Reading and Comprehension, is based on the theory of
argumentation from AI. It uses argumentation to capture appropriate solutions to the frame, ramification and
qualification problems and generalizations of these problems required
for text comprehension. In this first
part of the study the work concentrates on the central problem of
integration (or elaboration) of the explicit information from the narrative in the text
with the implicit (in the reader's mind) common sense world knowledge
pertaining to the topic(s) of the story given in the text.
We also report on our empirical efforts to gather background
common sense world knowledge used by humans when reading a story
and to evaluate, through a prototype system, the ability of our
approach to capture both the majority and the variability of
understanding of a story by the human readers in the experiments.


\end{abstract}

\section{Introduction}

Text comprehension has long been identified as a key test for Artificial Intelligence (AI). Aside from its central position in many forms of the Turing Test, it is clear that human computer interaction could benefit enormously from this and other forms of natural language processing. The rise of computing over the Internet, where so much data is in the form of textual information, has given even greater importance to this topic.
This paper reports on a research program aiming to learn from the (extensive) study of text comprehension in Psychology in order to draw guidelines for developing frameworks for automating narrative text comprehension and in particular, \textit{story comprehension} (SC).

Our research program brings together knowhow from Psychology and AI, in particular, our understanding of Reasoning about Actions and Change and Argumentation in AI, to provide a formal framework of representation and a computational framework for SC, that can be empirically evaluated and iteratively developed given the results of the evaluation. This empirical evaluation, which forms an important part of the program, is based on the following methodology: (i) set up a set of stories and a set of questions to test different aspects of story comprehension; (ii) harness the world knowledge on which human readers base their comprehension; (iii) use this world knowledge in our framework and automated system and compare its comprehension behaviour with that of the source of the world knowledge.

In this paper we will concentrate on the development of an appropriate Reasoning about Actions and Change and Default Reasoning framework for representing narratives extracted from stories together with the background \emph{world knowledge} needed for the underlying central process for story comprehension of \emph{synthesizing and elaborating} the explicit text information with new inferences through the implicit world knowledge of the reader.
In order to place this specific consideration in the overall process of story comprehension we
present here a brief summary of the problem of story comprehension from the psychological point of view.

\subsection{A Psychological Account of Story Comprehension}

Comprehending text entails the construction of a mental representation of the information contained in the text. However, no text specifies clearly and completely all the implications of text ideas or the relations between them. Therefore, comprehension depends on the ability to mentally represent the text-given information and to generate \textbf{bridging and elaborative inferences} that connect and elaborate text ideas resulting in a mental or \textbf{comprehension model} of the story. Inference generation is necessary in order to comprehend any text as a whole, i.e., as a single network of interconnected propositions instead of as a series of isolated sentences, and to appreciate the suspense and surprise that characterize narrative texts or stories, in particular \cite{BrewerLichtenstein1982,McNamaraMagliano2009}.

Although inference generation is based on the activation of background \textbf{world knowledge}, the process is constrained by text information. Concepts encountered in the text activate related conceptual knowledge in the readers' long-term memory \cite{Kintsch1988}. In the case of stories, knowledge about mental states, emotions, and motivations is also relevant as the events depicted tend to revolve around them.
Nevertheless, at any given point in the process, only a small subset of all the possible knowledge-based inferences remain activated and become part of the mental representation: those that connect and elaborate text information in a way that contributes to the \textbf{coherence} of the mental model \cite{McNamaraMagliano2009,RappVanDenBroek2005}. Inference generation is a task-oriented process that follows the principle of \textbf{cognitive economy} enforced by a limited-resource cognitive system.

However, the results of this coherence-driven selection mechanism can easily exceed the limited working memory capacity of the human cognitive system. Therefore, coherence on a more global level is achieved through \textbf{higher-level integration} processes that operate to create macro-propositions that generalize or subsume a number of text-encountered concepts and the inferences that connected them. In the process, previously selected information that maintains few connections to other information is dropped from the mental model. This results in a more consolidated network of propositions that serves as the new anchor for processing subsequent text information \cite{Kintsch1998}.

Comprehension also requires an iterative general \textbf{revision mechanism} of the mental model that readers construct. The feelings of suspense and surprise that stories aim to create are achieved through discontinuities or changes (in settings, motivations, actions, or consequences) that are not predictable or are wrongly predictable solely on the basis of the mental model created so far. Knowledge about the structure and the function of stories leads readers to expect discontinuities and to use them as triggers to revise their mental model \cite{Z1994}. Therefore, a change in time or setting in the text may serve as a clue for revising parts of the mental model while other parts remain and integrated with subsequent text information.

The interaction of bottom-up and top-down processes for the purposes of coherence carries the possibility of \textbf{different} but equally legitimate or \textbf{successful comprehension} outcomes. Qualitative and quantitative differences in conceptual and mental state knowledge can give rise to differences between the mental models constructed by different readers. Nevertheless, comprehension is successful if these are primarily differences in elaboration but not in the level of coherence of the final mental model.

In this paper we will focus on the underlying lower-level task of constructing the possibly additional elements of the comprehension model and the process of revising these elements as the story unfolds with only a limited concern on the global requirements of coherence and cognitive economy. Our working hypothesis is that these higher level features of comprehension can be tackled on top of the underlying framework that we are developing in this paper, either at the level of the representational structures and language or with additional computational processes on top of the underlying computational framework defined in this paper. We are also
assuming as solved the orthogonal issue of correctly parsing the natural language of the text into some information-
equivalent structured (e.g., logical) form that gives us the explicit narrative of the story. This is not to say that this issue is not an important element of narrative text comprehension. Indeed, it may need to be tackled in conjunction with the problems on which we are focusing (since, for example, the problem of de-referencing pronoun and article anaphora could depend on background world knowledge and hence possibly on the higher-level whole comprehension of the text \cite{Winograd2012}.

In the next two sections we will develop an appropriate representation framework using preference based argumentation that enables us to address well all the three major problems of \emph{frame, ramification and qualification} and provide an associated revision process. The implementation of a system discussed after this shows how psychologically-inspired story comprehension can proceed as a sequence of elaboration and revision. The paper then presents, using the empirical methodology suggested by research in psychology, our initial efforts to evaluate how closely the inferences drawn by our framework and system match those given by humans engaged in a story comprehension task.

The following story will be used as a running example.

\noindent
\textit{\textbf{Story:} It was the night of Christmas Eve. After feeding the animals and cleaning the barn,
Papa Joe took his shotgun from above the fireplace and sat out on the porch
cleaning it. He had had this shotgun since he was young, and it had never failed him,
always making a loud noise when it fired.}

\textit{Papa Joe woke up early at dawn, picked up his shotgun and went off to forest.
He walked for hours, until the sight of two turkeys in the distance made him
stop suddenly. A bird on a tree nearby was cheerfully chirping away,
building its nest. He aimed at the first turkey, and pulled the trigger.}

\textit{After a moment's thought, he opened his shotgun and saw there were no bullets in the shotgun's chamber.
He loaded his shotgun, aimed at the turkey and pulled the trigger again.
Undisturbed, the bird nearby continued to chirp and build its nest.
Papa Joe was very confused. Would this be the first time that his shotgun had let him down?}

The story above along with other stories and material used for the
evaluation of our approach can be found at \texttt{http://cognition.ouc.ac.cy/narrative/}.

\section{KRR for Story Comprehension}

We will use methods and results from Argumentation
Theory in AI (e.g., \cite{Dung95,aspic2012}) and its links to the area of
Reasoning about Action and Change (RAC) with
Default Reasoning on the static properties
of domains (see~\cite{HarmelenEtAl2008} for an overview) to develop a
Knowledge Representation and Reasoning (KRR)
framework suitable for Story Comprehension (SC).
Our central premise is that SC can be formalized in
terms of argumentation accounting
for the qualification and
the revision of the inferences drawn as we
read a story.

The psychological research and understanding of SC
 will guide us in the way we exploit the
know how from AI. The close link between
human common sense reasoning, such as that for SC,
and argumentation has been recently re-enforced
by new psychological evidence \cite{sperber}
suggesting that human reasoning
is in its general form inherently argumentative.
In our proposed approach of KRR for SC
the reasoning to construct a comprehension model
and its qualification at all levels as the story unfolds
will be captured through a
uniform acceptability requirement on the arguments
that support the conclusions in the model.

The significance of this form of representation for SC is that
it makes easy the elaboration of new inferences from the explicit
information in the narrative, that, as we discussed in the
introduction, is crucially necessary for the successful comprehension of stories.
On the other hand, this easy form of elaboration
and the extreme form of qualification that it needs
can be mitigated by the requirement, again given from
the psychological perspective, that elaborative inferences
need to be \textit{grounded} on the narrative and
\textit{sceptical} in nature.
In other words, the psychological perspective of SC, that
also suggests that story comprehension is a process of
 \textit{``fast thinking''}, leads us to depart from a
standard logical view of drawing conclusions
based on the truth in all (preferred) models.
Instead, the emphasis is turned on building one
\emph{grounded and well-founded} model from a
collection of solid or sceptical
properties that are grounded on the text and
follow as unqualified conclusions.

We use a typical RAC language of Fluents, Actions, Times, with an extra sort of \notion{Actors}. An actor-action pair is an \notion{event}, and a fluent/event or its negation is a \notion{literal}. For this paper it suffices to represent times as natural numbers\footnote{In general, abstract time points called \notion{scenes} are useful.}
and to assume that time-points are dense between story elements to allow for the realization of indirect effects.
Arguments will be build from premises in the knowledge connected to any given story.
We will have three types of such knowledge units as premises or basic units of arguments.

\begin{definition}
 Let $L$ be a fluent literal, $X$ a fluent/event literal and $S$ a set of fluent/event literals.
 A \notion{unit argument or premise} has one of following forms:
 \begin{itemize}

 \item \notion{a unit property argument} $\pro{X}{S}$ or $\prc{X}{S}$;

 \item \notion{a unit causal argument} $\cau{X}{S}$;

 \item \notion{a unit persistence argument} $\per{L}{\set{L}}$ (which we sometimes write as $\persimple{L}$).

 \end{itemize}

 \noindent These three forms are called \notion{types} of unit arguments.
 A unit argument of any type is denoted by $\arg[i]{H_i}{B_i}$. The two forms of unit property arguments differ in that $\pro{X}{S}$ relates properties to each other at the same time-point, whereas $\prc{X}{S}$ aims to capture preconditions that hold at the time-point of an event, under which the event is blocked from bringing about its effects at the subsequent time-point.
\end{definition}

With abuse of terminology we will sometimes call these units of arguments, simply as arguments.

The knowledge required for the comprehension of a story comprises of two parts:
the explicit knowledge of the narrative extracted from the text of the story
and the implicit background knowledge that the reader uses along with the narrative
for elaborative inferences about the story.

\begin{definition}
A \notion{world knowledge theory} $\wk$ is a set of unit property and causal arguments together with a (partial) irreflexive priority relation on them. A \notion{narrative} $\narr$ is: a set of observations $\obs{X}{T}$ for a fluent/event literal $X$, and a time-point $T$; together with a (possibly empty)
set of (story specific) property or causal unit arguments.
\end{definition}

The priority relation in $\wk$ would typically
reflect the priority of specificity for properties,
expressed by unit property arguments $\pro{X}{S}$, or the priority of
precondition properties, expressed by unit property arguments $\prc{X}{S}$, over causal effects, expressed by unit causal arguments.
This priority amongst these basic units of
knowledge gives a form of non-monotonic reasoning (NMR) for deriving
new properties that hold in the story.

To formalize this NMR we use a form of preference-based argumentation uniformly to
capture the static (default) inference of properties at a single time point
as well as inferences between different type points,
by extending the domain specific priority relation to address the frame problem.

\begin{definition}
A \notion{story representation} $\story = \tuple{\wk,\narr,\succ}$ comprises a world knowledge theory $\wk$, a narrative $\narr$, and a (partial) irreflexive priority relation $\succ$ extending the one in $\wk$ so that: \cond{i} $\cau{H}{B_1} \succ \per{\neg H}{B_2}$; \cond{ii} $\per{H}{B_1} \succ \pro{\neg H}{B_2}$.
The extended relation $\succ$ may also prioritize between arguments in $\narr$ and those in $\wk$
(typically the former over the latter).
\end{definition}

The first priority condition, namely that causal arguments have
priority over persistence arguments, encompasses a solution to the \notion{frame problem}.
When we need to reason with defeasible
property information,
such as default rules about the normal
state of the world in which a story takes place,
we are also faced with a \notion{generalized frame problem},
where ``a state of the world persists irrespective of the
existence of general state laws''. Hence, if we are told
that the world is in fact in some exceptional state that
violates a general (default) property this will continue
to be the case in the future, until we learn of (or derive)
some causal information that returns the world into
its normal state. The solution to this
generalized frame problem is captured succinctly by the second general
condition on the priority relation of a story representation
and its combination with the first condition.

A representation $\story$ of our example story (focusing on its ending) may include the following unit arguments in $\wk$ and $\narr$ (where $pj$ is short for ``Papa Joe''):

\smallskip

\noindent{\small
$c1: \ \cau{fired\_at(pj,X)}{\set{aim(pj,X),pull\_trigger(pj)}}$\\
$c2: \ \cau{\neg alive(X)}{\set{fired\_at(pj,X),alive(X)}}$\\
$c3: \ \cau{noise}{\set{fired\_at(pj,X)}}$\\
$c4: \ \cau{\neg chirp(bird)}{\set{noise,nearby(bird)}}$\\
$c5: \ \cau{gun\_loaded}{\set{load\_gun}}$\\
$p1: \ \prc{\neg fired\_at(pj,X)}{\set{\neg gun\_loaded}}$\\
$p2: \ \pro{\neg fired\_at(pj,X)}{\set{\neg noise}}$ \hfill (story specific) \\
}

\smallskip

\noindent with $p1 \succ c1$, $p2 \succ c1$; and the following in $\narr$:

\smallskip

\noindent $\obs{alive(turkey)}{1}$, $\obs{aim(pj,turkey)}{1}$, $\obs{pull\_trigger(pj)}{1}$, $\obs{\neg gun\_loaded}{4}$, $\obs{load\_gun}{5}$, $\obs{pull\_trigger(pj)}{6}$, $\obs{chirp(bird)}{10}$, \hfill $\obs{nearby(bird)}{10}$,

\smallskip

\noindent with the exact time-point choices being inconsequential.

As we can see in this example the representation of common
sense world knowledge has the form of simple associations
between concepts in the language. This stems from a
key observation in psychology that typically all world knowledge
and irrespective of type is inherently default.
It is not in the form of an elaborate formal theory
of detailed definitions of concepts, but rather is better regarded
as a collection of relatively loose semantic associations
between concepts, reflecting typical rather than absolute information.
Thus knowledge need not be fully qualified at the representation level,
since it can be qualified via the reasoning process
by the relative strength of other (conflicting) associations
in the knowledge. In particular, as we will see below,
\textbf{endogenous qualification} will be tackled by the priority
relation in the theory and exogenous qualification
by this priority coupled with the requirement that
explicit narrative information forms, in effect,
non-defeasible arguments.

\subsection{Argumentation Semantics for Stories}

To give the semantics of any given story representation $\story$
we will formulate a corresponding
preference based argumentation
framework of the form $\tuple{Arguments, Disputes, Defences}$.
Arguments will be based
on sets of timed unit arguments.
Since we are required to reason about properties
over time, it is necessary that arguments populate
some connected subset of the time line.

\begin{definition}
Let $\story = \tuple{\wk,\narr,\succ}$ be a story representation.
A \notion{(unit) argument tuple} has the form
$\tuple{\arg{H}{B},T^h,d;(X,T)}$, where, $\arg{H}{B}$,
is a unit argument in $\story$, $X$ is a fluent/event literal,
$d \in \set{\f,\b}$ is an \notion{inference type} of either
forwards derivation or backwards derivation
by contradiction, and $T^h,T$ are time points.
$T^h$ refers to the time-point at which the head of the unit argument applies,
while $X$ and $T$ refer to the conclusion drawn using the unit argument
in the tuple.
An \notion{interpretation} $\Delta$ of $\story$ is then defined as a
a set of argument tuples.
We say $\Delta$ \notion{supports} a fluent/event literal, $X$, \notaux{at} $T$, if either $\tuple{\arg{H}{B},T^h,d;(X,T)} \in \Delta$ or $\obs{X}{T} \in \narr$. The notion of support is extended to hold on sets of timed literals.
\end{definition}

The inference process of how an argument tuple
supports a timed literal,
and thus is allowed to belong to an interpretation,
is made precise by the following definition.

\begin{definition}
Let $\Delta$ be an interpretation and $\tuple{\arg{H}{B},T^h,d;(X,T)}$ in $\Delta$ with $d=F$.
Then $\arg{H}{B}$ \notaux{applied at} $T^h$ \notion{forward derives} $X$ \notaux{at} $T$ \notaux{under} $\Delta$ iff $X = H$, $T=T^h$ and $\Delta$ supports $B$ at $T'$. The set $\set{\tuple{Y,T'}\mid Y \in B}$ is called the \notion{activation condition} for the derivation; $T' = T^h$ if $\arg{H}{B}$ is of the form $\pro{H}{B}$. $T'= T^h-1$ for the other argument types.\\
When $d=B$, $\arg{H}{B}$ \notaux{applied at} $T^h$ \notion{backward derives} $X$ \notaux{at} $T$ \notaux{under} $\Delta$ iff $\neg X \in B$ and $\Delta$ supports $\set{\neg H}$ at $T^h$ and $B \setminus \set{\neg X}$ at $T$. The set $\set{\tuple{\neg H,T^h}} \cup \set{\tuple{Y,T}\mid Y \in B \setminus \set{\neg X}}$ is the \notion{activation condition}; $T = T^h$ if $\arg{H}{B}$ is of the form $\pro{H}{B}$. $T= T^h-1$ for the other argument types.

\end{definition}

The framework thus includes
\notion{reasoning by contradiction}
with the defeasible world knowledge.
Although the psychological debate on the question to what extent
humans reason by contradiction, e.g., by contraposition,
(see, e.g., \cite{JLY08,Rips94}) is still ongoing it is natural
for a formal argumentation framework to capture this mode
of indirect reasoning (see, e.g., \cite{kmtcs13,kmjlc13}).
One of the main consequences of this is that
it gives a form of \notion{backwards persistence}, e.g.,
from an observation to support (but not necessarily conclude)
that the observed property holds also at previous
time points. An argument tuple of the form $\tuple{\persimple{L},T+1,\b;(\neg L, T)}$
captures the backwards persistence of $\neg L$ from time $T+1$ to $T$
using by contraposition the unit argument of persistence of $L$ from $T$ to $T+1$.
We also note that the separation of the inference type (e.g., forwards
and backwards) is known to be significant in
preference based argumentation \cite{aspic2012}. This will
be exploited when we consider the attacking between
arguments: their disputes and defences.

To reflect the suggestion by psychology
that inferences drawn by readers are
strongly tied to the story we require
that the activation conditions of argument tuples
must be eventually traced on the explicit information
in the narrative of the story representation.

\begin{definition} An interpretation $\Delta$ is \notion{grounded} on $\story$ iff there is a total ordering of $\Delta$ such that the activation condition of any tuple $\alpha \in \Delta$ is supported by the set of tuples that precede $\alpha$ in the ordering or by the narrative in $\story$.
\end{definition}

Hence in a grounded interpretation there can be no cycles in the tuples
that support their activation conditions and so these
will always end with tuples whose activation conditions will be supported
directly by the observations in the narrative of the story.

We can now define the argumentation framework corresponding to
any given story representation. The central task is to
capture through the argumentation semantics the non-monotonic
reasoning of linking the narrative to the defeasible
information in the world knowledge.  In particular, the
argumentation will need to capture the
\notion{qualification} problem, encompassed in this synthesis
of the narrative with the world knowledge,
both at the level of static reasoning at one time point
with default property arguments and at the level of temporal
projection from one time point to another.

\begin{definition}
Let $\story$ be a story representation. Then the
\notion{corresponding argumentation framework,
$\tuple{ARG^{\story},DIS^{\story}, DEF^{\story}}$} is
defined as follows:

\begin{itemize}

\item An argument, $A$, in $ARG^{\story}$ is any grounded
interpretation of $\story$.

\item Given an argument $A$ then $A$ is \notion{in conflict} with $\story$
iff there exists a tuple $\alpha = \tuple{\arg{H}{B},T^h,d;(X,T)}$ in $A$
such that $\obs{\neg X}{T} \in \narr$ of $\story$.

\item Given two arguments $A_1, A_2$ then these are \notion{in (direct) conflict} with each other iff
there exists a tuple $\alpha_2 = \tuple{\arg[2]{H_2}{B_2},T^h_2,d_2;(X_2,T_2)}$ in $A_2$
and a tuple $\alpha_1 = \tuple{\arg[1]{H_1}{B_1},T^h_1,d_1;(X_1,T_1)}$ in $A_1$
such that $X_1 = \neg X_2$, $T_1 = T_2$. Given two arguments $A_1, A_2$ then these are \notion{in indirect conflict} with each other iff
there exists a tuple $\alpha_2 = \tuple{\arg[2]{H_2}{B_2},T^h_2,d_2;(X_2,T_2)}$ in $A_2$
and a tuple $\alpha_1 = \tuple{\arg[1]{H_1}{B_1},T^h_1,d_1;(X_1,T_1)}$ in $A_1$
such that ($d_1 = \b$ or $d_2 = \b$) and $H_1 = \neg H_2$, $T^h_1 = T^h_2$.

\item Given two arguments $A_1, A_2$ then $A_2$ \notion{disputes} $A_1$ and
hence $(A_2,A_1) \in DIS^{\story}$ iff $A_2$ is in direct or indirect
conflict with $A_1$, and in the case of indirect conflict $d_1 = \b$ holds in the definition of indirect conflict above.

\item Argument $A_1$ \notion{undercuts} $A_2$ iff

\begin{itemize}
\item $A_1, A_2$ are in direct or indirect conflict
via $\alpha_1$ and $\alpha_2$,

\item when in direct conflict, there exists a tuple $\alpha_1' = \tuple{\argsub[_1']{H_1'}{B_1'},T^{h'}_1,d_1';(X_1',T_1')}$ in $A_1$
and a tuple $\alpha_2' = \tuple{\argsub[_2']{H_2'}{B_2'},T^{h'}_2,d_2';(X_2',T_2')}$ in
$A_2$ such that $\argsub[_1']{H_1'}{B_1'} \succ \argsub[_2']{H_2'}{B_2'}$ and
$T^{'}_1=T^{'}_2$ or $T^{h'}_1=T^{h'}_2$.

\item when in indirect conflict, then $\arg[1]{H_1}{B_1} \succ \arg[2]{H_2}{B_2}$
where $\arg[1]{H_1}{B_1}$ and $\arg[2]{H_2}{B_2}$ are the unit arguments
in $\alpha_1$ and $\alpha_2$ respectively.
\end{itemize}

\item Argument $A_1$ \notion{defends against} $A_2$ and hence $(A_1,A_2) \in DEF^{\story}$, iff
there exists a subset $A_2^{'} \subseteq A_2$ which is in minimal conflict with $A_1$
(i.e., no proper subset of $A_2^{'}$ is in conflict with $A_1$) and $A_1$ undercuts $A_2^{'}$.
\end{itemize}

\end{definition}

Several clarifying comments are in order.
Arguments that are in dispute are arguments that support some contrary conclusion at the
same time point and hence form counter-arguments for each other.
The use of contrapositive reasoning for backwards inference also means that
it is possible to have arguments that support
conclusions that are not contrary to each other but
whose unit arguments
have conflicting conclusions. For example,
in our running example we can use the causal unit
argument, $c1$, to forward derive $fired\_at(pj,X)$ and
the property argument $p1$ to backwards derive
$gun\_loaded$ from $\neg fired\_at(pj,X)$ and despite the
fact that the derived facts are not in conflict the
unit arguments used concern conflicting conclusions.
Hence such arguments are also considered to be in conflict
but instead of a direct conflict we say we have
an indirect conflict. Not all such indirect conflicts are important.
A dispute that results from an indirect conflict of a unit argument
used backwards on a unit argument that is used forwards does
not have any effect. Such cases are excluded from giving rise to disputes.

This complication in the definitions of conflicts and disputes results from the \emph{defeasible nature} of the world
knowledge and the fact we are
allowing reasoning by contradiction on such defeasible information.
These complications in fact stem from the fact that we are only
approximating the proof by contradiction reasoning, capturing this
indirectly through contraposition.
The study of this is beyond the
scope of this paper and the reader is referred to
the newly formulated Argumentation Logic
\cite{kmtcs13}.

Undercuts between arguments require that the undercutting
argument does so through a stronger unit or premise argument
than some unit argument in the argument that is undercut.
The defence relation is build out of undercuts
by applying an undercut on minimally conflicting
subsets of the argument which we are defending against.
Hence these two relations between arguments are asymmetric.
Note also that the stronger premise from the undercutting argument
does not necessarily need to come from the subset of the unit arguments
that supports the conflicting conclusion. Instead, it can come
from any part of the undercutting argument to undercut at
any point of the chain supporting the activation of the conflicting
conclusion. This, as we shall
illustrate below, is linked to how the framework addresses the
\notion{ramification problem} of
reasoning with actions and change.

The semantics of a story representation
is defined using the corresponding argumentation
framework as follows.

\begin{definition}
Let $\story$ be a story representation and
$\tuple{ARG^{\story},DIS^{\story}, DEF^{\story}}$ its corresponding argumentation
framework. An argument $\Delta$ is \notion{acceptable} in $\story$ iff

\begin{itemize}

\item $\Delta$ is not in conflict with $\story$ nor in direct conflict with $\Delta$.

\item No argument $A$ undercuts $\Delta$.

\item For any argument $A$ that minimally disputes $\Delta$, $\Delta$ defends against $A$.

\end{itemize}

Acceptable arguments are called
\notion{comprehension models} of $\story$.
Given a comprehension model $\Delta$, a timed fluent
literal $(X,T)$ is \notion{entailed by $\story$} iff
this is supported by $\Delta$.
\end{definition}

The above definition of comprehension model and story
entailment is of a sceptical form where, apart from
the fact that all conclusions must be ground on the
narrative, they must also not be non-deterministic in the
sense that there can not exist another comprehension model where
the negative conclusion is entailed.
Separating disputes and undercuts and identifying
defences with undercuts facilitates
this sceptical form of entailment.
Undercuts (see, e.g., \cite{aspic2012} for some recent discussion)
are strong counter-claims whose existence means that
the attacked set is inappropriate for
sceptical conclusions whereas disputes
are weak counter-claims that could be defended or
invalidated by extending the argument to undercut them back.
Also the explicit condition that an acceptable
argument should not be undercut
even if it can undercut back means that this definition
does not allow non-deterministic choices for arguments
that can defend themselves.

To illustrate the formal framework, how arguments are constructed
and how a comprehension of a story is formed through acceptable arguments
let us consider our example story starting
from the end of the second paragraph,
corresponding to time-points $1$-$3$ in the example narrative.
Note that the empty $\Delta$ supports $aim(pj,turkey)$ and $pull\_trigger(pj)$ at $1$.
Hence, $c1$ on $2$ forward activates $fired\_at(pj,turkey)$ at $2$ under the empty argument, $\Delta$.
 We can thus populate $\Delta$ with $\tuple{c1,2,\f;(fired\_at(pj,turkey),2)}$. Similarly, we can include $\tuple{\persimple{alive(turkey)},2,\f;(alive(turkey),2)}$ in the new $\Delta$. Under this latter $\Delta$, $c2$ on $3$ forward activates $\neg alive(turkey)$ at $3$, allowing us to further extend $\Delta$ with $\tuple{c2,3,\f;(\neg alive(turkey),3)}$. The resulting $\Delta$ is a grounded interpretation that supports $\neg alive(turkey)$ at $3$. It is based on this inference, that we expect readers to respond that the first turkey is dead, when asked about its status at this point, since no other argument grounded on the narrative (thus far) can support a qualification argument to this inference. Note also that we can include in $\Delta$
 the tuple $\tuple{p1,2,\b;(gun\_loaded,1)}$ to support, using backwards
 (contrapositive) reasoning with $p1$, the conclusion that the gun was loaded
 when it was fired at time $1$.

 Reading the first sentence of the third paragraph, we learn that $\obs{\neg gun\_loaded}{4}$. We now expect that this new piece of evidence will lead readers to revise their inferences as now we have an argument to support the conclusion $\neg fired\_at(pj, turkey)$ based on the stronger (qualifying) unit argument of $p1$. For this we need to support the activation condition of $p1$ at time $1$, i.e., to support $\neg gun\_loaded$ at $1$. To do this we can use the argument tuples:\\
\indent $\tuple{\persimple{gun\_loaded},4,\b;(\neg gun\_loaded, 3)}$\\
\indent $\tuple{\persimple{gun\_loaded},3,\b;(\neg gun\_loaded, 2)}$\\
\indent $\tuple{\persimple{gun\_loaded},2,\b;(\neg gun\_loaded, 1)}$\\
which support the conclusion that the gun was also unloaded before it was observed to be so. This uses $\persimple{gun\_loaded}$ contrapositively to backward activate the unit argument of persistence, e.g.,
had the gun been loaded at $3$,
it would have been so at $4$
which would contradict the story.
Note that this backwards inference of
$\neg gun\_loaded$ would be
qualified by a causal argument for
$\neg gun\_loaded$ at any time earlier than $4$, e.g.,
if the world knowledge contained the unit argument

$c: \ \cau{\neg gun\_loaded}{\set{pull\_trigger(pj)}}$

\noindent This then supports an indirect conflict at time $2$
with the forwards persistence of $gun\_loaded$ from $1$ to $2$ and
due to the stronger nature of unit causal over persistence arguments
the backwards inference of $\neg gun\_loaded$ is undercut and so cannot belong to
an acceptable argument.

Assuming that $c$ is absent, the argument, $\Delta_1$, consisting of these three
``persistence'' tuples is in conflict on $gun\_loaded$ on $1$
with the argument $\Delta$ above.
Each argument disputes the other and in fact neither
can form an acceptable argument. If we extend
$\Delta_1$ with the tuple $\tuple{p1,2,\f;(\neg fired\_at(pj,turkey),2)}$
then this can now undercut and thus defend against
$\Delta$ using the priority of $p1$ over $c1$.
Therefore the extended $\Delta_1$ is acceptable
and the conclusion $\neg fired\_at(pj,turkey)$ at $2$
is drawn revising the previous conclusions drawn
from $\Delta$.
The process of understanding our story may then proceed by extending $\Delta_1$, with  $\tuple{\persimple{alive(turkey)},T,\f;(alive(turkey),T)}$ for $T=2,3,4$, resulting in a model that supports $alive(turkey)$ at $4$.
It is based on this inference that we expect readers to respond that the first turkey is alive at $4$.

Continuing with the story,
after Papa Joe loads the gun and fires again, we can support by forward inferences that the gun fired, that noise was caused, and that the bird stopped chirping, through a chaining of the unit arguments $c1,c3,c4$. But $\obs{chirp(bird)}{10}$ supports disputes on all these through the repeated backwards use of the same unit arguments grounded on this observation. We thus have an \textbf{exogenous qualification} effect where these conclusions can not be sceptical and so will not be supported by any comprehension model.
But if we also consider the stronger (story specific) information in $p2$, that this
gun does not fire without a noise, together with the backwards inference of $\neg noise$
an argument that contains these can undercut the firing of the gun at time $2$
and thus defend against disputes
that are grounded on $pull\_triger$ at $1$ and
the gun firing.
As a result, we have the effect of blocking the ramification of the causation of $noise$ and so $\neg noise$ (as well as $\neg fired\_at(pj,turkey)$) are sceptically concluded. Readers, indeed respond in this way.

With this latter part of the example story we see how our framework addresses the
\emph{ramification problem} and its non-trivial interaction
with the qualification problem \cite{ThielscherQual01}.
In fact, a generalized form of this problem is addressed
where the ramifications are not chained only through causal laws but through
any of the forms of inference we have in the framework ---
causal, property or persistence --- and through any of the type of
inference --- forwards or backwards by contradiction.

A comprehension model can be tested, as is often done in psychology,
through a series of multiple-choice questions.

\begin{definition}
Let $M$ be a comprehension model of a story representation
$\story$. A possible answer,``$X$ at $T$'',
to a question is \notion{accepted},
respectively \notion{rejected},
iff ``$X$ at $T$'' (respectively ``$\neg X$ at $T$'')
is supported by $M$. Otherwise, we say that
the question is \notion{allowed or possible} by $M$.
\end{definition}

In some cases, we may want to extend the notion
of a comprehension model to allow some non-sceptical
entailments. This is needed to reflect
the situation when a reader cannot find a sceptical answer
to a question and chooses between
two or more allowed answers. This can be captured
by allowing each such answer to be supported by a more general notion
of acceptability such as the admissibility
criterion of argumentation semantics.
For this, we can drop the condition that $\Delta$ is not undercut by any argument and
allow weaker defences, through disputes, to defend back
on a dispute that is not at the same time an undercut.

Finally, we note that a comprehension
model need not be complete as it
does not need to contain all possible
sceptical conclusions that can be drawn from the
narrative and the entire world knowledge. It is a
\emph{subset} of this,
given by the subset of the available world knowledge
that readers choose to use.
This incompleteness of the comprehension model is
required for important cognitive
economy and coherence properties of
comprehension, as trivially a ``full
model'' is contrary to the notion of coherence.

\section{Computing Comprehension Models}

\def\D{\Delta}
\def\P{\Pi}
\def\s{\story}
\def\E{E}

The computational procedure below constructs a comprehension model, by iteratively reading a new part of the story $\s$, retracting existing inferences that are no longer appropriate, and including new inferences that are triggered as a result of the new story part. Each part of the story may include more than one observation, much in the same way that human readers may be asked to read multiple sentences in the story before being asked to answer a question. We shall call each story part of interest a \notion{block}, and shall assume that it is provided as input to the computational procedure.

\begin{algorithm}[t]
\caption{Computing a Comprehension Model}
\label{Algorithm: Computing a Comprehension Model}%
\begin{algorithmic}
\STATE \textbf{input:} story $\s$, partitioned in a list of $k$ blocks, and a set of questions $Q[b]$ associated with each $\s$ block $b$.
\STATE Set $\G[0]$ to be the empty graph.
\FOR{every $b = 1, 2, \ldots, k$}
    \STATE Let $\s[b]$ be the restriction of $\s$ up to its $b$-th block.
    \STATE Let $\G[b] := graph(\G[b-1],\s[b])$ be the new graph.
    \STATE Let $\P[b] := retract(\D[b-1],\G[b],\s[b])$.
    \STATE Let $\D[b] := elaborate(\P[b],\G[b],\s[b])$.
    \STATE Answer $Q[b]$ with the comprehension model $\D[b]$.
\ENDFOR
\end{algorithmic}
\end{algorithm}

At a high level the procedure proceeds as in Algorithm~\ref{Algorithm: Computing a Comprehension Model}. The story is read one block at a time. After each block of $\s$ is read, a directed acyclic graph $\G[b]$ is maintained which succinctly encodes all interpretations that are relevant for $\s$ up to its $b$-th block. Starting from $\G[b-1]$, a new tuple is added as a vertex if it is possible to add a directed edge to each $\tuple{X,T}$ in the tuple's condition from either an observation $\obs{X}{T}$ in the narrative of $\s[b]$, or from a tuple $\tuple{\arg{H}{B},T^h,d;(X,T)}$ already in $\G[b]$. In effect, then, edges correspond to the notion of support from the preceding section, and the graph is the maximal grounded interpretation given the part of the story read.

Once graph $\G[b]$ is computed, it is used to revise the comprehension model $\D[b-1]$ so that it takes into account the observations in $\s[b]$. The revision proceeds in two steps.

In the first step, the tuples in $\D[b-1]$ are considered in the order in which they were added, and each one is checked to see whether it should remain in the comprehension model. Any tuple in $\D[b-1]$ that is undercut by the tuples in $\G[b]$, or disputed and cannot be defended, is retracted, and is not included in the provisional set $\P[b]$. As a result of a retraction, any tuple $\tuple{\arg{H}{B},T^h,d;(X,T)} \in \D[b-1]$ such that $\arg{H}{B}$ no longer activates $X$ at $T$ under $\P[b]$ is also retracted and is not included in $\P[b]$. This step guarantees that the argument $\P[b]$ is trivially acceptable.

In the second step, the provisional set $\P[b]$, which is itself a comprehension model (but likely a highly incomplete one), is elaborated with new inferences that follow. The elaboration process proceeds as in Algorithm~\ref{Algorithm: Elaborating a Comprehension Model}. Since the provisional comprehension model $\P$ effectively includes only unit arguments that are ``strong'' against the attacks from $\G$, it is used to remove (only as part of the local computation of this procedure) any weak arguments from $\G$ itself (i.e., arguments that are undercut), and any arguments that depend on the former to activate their inferences. This step, then, ensures that all arguments (subsets of $G$) that are defended are no longer part of the revised $G$, in effect accommodating the minimality condition for attacking sets. It then considers all arguments that activate their inferences in the provisional comprehension model. The comprehension model is expanded with a new tuple from $\E$ if the tuple is not in conflict with the story nor in direct conflict with the current model $\D$, and if ``attacked'' by arguments in $\G$ then these arguments do not undercut $\D$, and $\D$ undercuts back. Only arguments coming from the \emph{revised} graph $\G$ are considered, as per the minimality criterion on considered attacks.

\begin{algorithm}[t]
\caption{Elaborating a Comprehension Model}
\label{Algorithm: Elaborating a Comprehension Model}%
\begin{algorithmic}
\STATE \textbf{input:} provisional comprehension model $\P$, graph $\G$, story $\s$; all inputs possibly restricted up to some block.
\REPEAT
    \STATE Let $\G := retract(\P,\G,\s)$.
    \STATE Let $\E$ include all tuples $\tuple{\arg{H}{B},T^h,d;(X,T)}$\\ ~~~~~~such that $\arg{H}{B}$ activates $X$ at $T$ under $\P$.
    \STATE Let $\P := \D$.
    \STATE Let $\P := expand(\D,\E,\G)$.
    \UNTIL{$\D = \P$}
\STATE \textbf{output:} elaborated comprehension model $\D$.
\end{algorithmic}
\end{algorithm}

The elaboration process adds only ``strong'' arguments in the comprehension model, retaining its property as a comprehension model. The discussion above forms the basis for the proof of the following theorem:

\begin{theorem}
Algorithm~\ref{Algorithm: Computing a Comprehension Model} runs in time that is polynomial in the size of $\s$ and the number of time-points of interest, and returns a comprehension model of the story.
\end{theorem}

\begin{proof}[Proof sketch]
Correctness follows from our earlier discussion. Regarding running time:
The number of iterations of the top-level algorithm is at most linear in the relevant parameters. In constructing the graph $\G[b]$, each pair of elements (unit arguments or observations at some time-point) in $\s[b]$ is considered once, for a constant number of operations. The same is the case for the retraction process in the subsequent step of the algorithm. Finally, the loop of the elaboration process repeats at most a linear in the relevant parameters number of times, since at least one new tuple is included in $\P$ in every loop. Within each loop, each step considers each pair of elements (unit arguments or observations at some time-point) in $\s[b]$ once, for a constant number of operations.
The claim follows.
\end{proof}

The computational processes presented above have been implemented using Prolog, along with an accompanying high-level language for representing narratives, background knowledge, and multiple-choice questions. Without going into details, the language allows the user to specify a sequence of sessions of the form
\texttt{session(s(B),Qs,Vs)}, where \texttt{B} is the next story block to read, \texttt{Qs} is the set of questions to be answered afterwards, and \texttt{Vs} is the set of fluents made visible in a comprehension model returned to the user.

The narrative itself is represented by a sequence of statements of the form \texttt{s(B) :: X at T}, where \texttt{B} is the block in which the statement belongs (with possibly multiple statements belonging in the same block), \texttt{X} is a fluent or action, and \texttt{T} is the time-point at which it is observed.

The background knowledge is represented by clauses of the form \texttt{p(N) :: A, B, ..., C implies X}  or \texttt{c(N) :: A, B, ..., C causes X}, where \texttt{p} or \texttt{c} shows a property or causal clause, \texttt{N} is the name of the rule, \texttt{A, B, ..., C} is the rule's body, and \texttt{X} is the rule's head. Negations are represented by prefixing a fluent or action in the body or head with the minus symbol. Variables can be used in the fluents or actions to represent relational rules. Preferences between clauses are represented by statements of the form \texttt{p(N1) >> c(N2)} with the natural reading.

Questions are represented by clauses of the form \texttt{q(N) ?? (X1 at T1, ..., X2 at T2) ; ...}, where \texttt{N} is the name of the question, \texttt{(X1 at T1, ..., X2 at T2)} is the first possible answer as a conjunction of fluents or actions that need to hold at their respective time-points, and \texttt{;} separates the answers. The question is always the same: ``Which of the following choices is the case?''.

The implemented system demonstrates real modularity and elaboration tolerance, allowing as input any story narrative or background knowledge in the given syntax, always appropriately qualifying the given information to compute a comprehension model. The system is available at \texttt{http://cognition.ouc.ac.cy/narrative/}.

\section{Evaluation through Empirical Studies}

In the first part of the evaluation of our approach we carried a \emph{psychological study} to ascertain the world knowledge that is activated to successfully comprehend example stories such as our example story on the basis of data obtained from human readers. We were interested both in the outcomes of successful comprehension and the world knowledge that contributed to the human comprehension.
We developed a set of inferential questions to follow the reading of pre-specified story segments. These assessed the extent to which readers connected, explained, and elaborated key story elements. Readers were instructed to answer each question and to justify their answers using a ``think-aloud'' method of answering questions while reading in order to reveal the world knowledge that they had used.

The qualitative data from the readers was pooled together and analysed as to the frequencies of the types of responses in conjunction with the information given in justifications and think-aloud protocols.
For example, the data indicated that all readers considered Papa Joe to be living on a farm or in a village (q.01, ``Where does Papa Joe live?'') and that all readers attributed an intention of Papa Joe to hunt
(q.06, ``What was Papa Joe doing in the forest?'').
An interesting example of variability occurred in the answers
for the group of questions 07,08,10,11, asking about the status of the turkeys
at various stages in the story.
The majority of participants followed a comprehension model which
was revised between the first turkey being dead and alive.
However, a minority of participants consistently answered that both turkeys were alive. These readers had defeated the causal arguments that supported the inference that the first turkey was dead, perhaps based
on an expectation that the desire of the protagonist
for turkey would be met with complications.
We believe that such expectations can be generated from
standard \emph{story knowledge} in the same way as we draw other elaborative inferences from WK.

\subsection{Evaluation of the system}

\def\push{${}$~~~~~~~~~~~~~}

Using the empirical data discussed above, we tested our framework's ability to capture the majority answers and account for their variability. The parts of our example story representation relevant to questions 01 and 06 are as follows:

{\small
\begin{quote}
\begin{tabular}{ l l }
s(1) :: night at 0.              &     s(2) :: animal(turkey2) at 2.\\
s(1) :: xmasEve at 0.            &     s(2) :: alive(turkey1) at 2.\\
s(1) :: clean(pj,barn) at 0.     &     s(2) :: alive(turkey2) at 2.\\
s(2) :: xmasDay at 1.            &     s(2) :: chirp(bird) at 2.\\
s(2) :: gun(pjGun) at 1.         &     s(2) :: nearby(bird) at 2.\\
s(2) :: longWalk(pj) at 1.       &     s(2) :: aim(pjGun,turkey1) at 2.\\
s(2) :: animal(turkey1) at 2.    &     s(2) :: pulltrigger(pjGun) at 2.\\
\end{tabular}
\end{quote}
}

The two questions are answered after reading, respectively, the first and second blocks of the story above:

{\small
\begin{quote}
\begin{tabular}{ l l }
session(s(1),[q(01)],\_).   &      session(s(2),[q(06)],\_).
\end{tabular}
\end{quote}
}

\noindent with their corresponding multiple-choice answers being:

{\small
\begin{quote}
\begin{tabular}{ l l }
q(01) ?? & {\ }\\
lives(pj,city) at 0;    &    lives(pj,hotel) at 0;\\
lives(pj,farm) at 0;     &    lives(pj,village) at 0.\\
\end{tabular}

\begin{tabular}{ l l }
q(06) ?? & {\ }\\
motive(in(pj,forest),practiceShoot) at 3; & {\ }\\
motive(in(pj,forest),huntFor(food)) at 3; & {\ }\\
(motive(in(pj,forest),catch(turkey1)) at 3, & {\ }\\
\push motive(in(pj,forest),catch(turkey2)) at 3); & {\ }\\
motive(in(pj,forest),hearBirdsChirp) at 3. & {\ }\\
\end{tabular}
\end{quote}
}

To answer the first question, the system uses the following background knowledge:

{\small
\begin{quote}
\begin{tabular}{ l l }
p(11) :: has(home(pj),barn) implies lives(pj,countrySide). & {\ }\\
p(12) :: true implies -lives(pj,hotel). & {\ }\\
p(13) :: true implies lives(pj,city). & {\ }\\
p(14) :: has(home(pj),barn) implies -lives(pj,city). & {\ }\\
p(15) :: clean(pj,barn) implies at(pj,barn). & {\ }\\
p(16) :: at(pj,home), at(pj,barn) implies has(home(pj),barn). & {\ }\\
p(17) :: xmasEve, night implies at(pj,home). & {\ }\\
p(18) :: working(pj) implies -at(pj,home). & {\ }\\

p(111) :: lives(pj,countrySide) implies lives(pj,village). & {\ }\\
p(112) :: lives(pj,countrySide) implies lives(pj,farm). & {\ }\\
p(113) :: lives(pj,village) implies -lives(pj,farm). & {\ }\\
p(114) :: lives(pj,farm) implies -lives(pj,village). & {\ }\\
\end{tabular}

\begin{tabular}{ l l }
p(14) $>>$ p(13). &    p(18) $>>$ p(17).
\end{tabular}
\end{quote}
}

By the story information, p(17) implies at(pj,home), without being attacked by p(18), since nothing is said in the story about Papa Joe working. Also by the story information, p(15) implies at(pj,barn). Combining the inferences from above, p(16) implies has(home(pj),barn), and p(11) implies lives(pj,countrySide). p(12) immediately dismisses the case of living in a hotel (as people usually do not), whereas p(14) overrides p(13) and dismisses the case of living in the city. Yet, the background knowledge cannot unambiguously derive one of the remaining two answers. In fact, p(111), p(112), p(113), p(114) give arguments for either of the two choices. This is in line with the variability in the empirical data in terms of human answers to the first question.

To answer the second question, the system uses the following background knowledge:

{\small
\begin{quote}
p(21) :: want(pj,foodFor(dinner)) implies\\
\push motive(in(pj,forest),huntFor(food)).

p(22) :: hunter(pj) implies\\
\push motive(in(pj,forest),huntFor(food)).

p(23) :: firedat(pjGun,X), animal(X) implies\\
\push -motive(in(pj,forest),catch(X)).

p(24) :: firedat(pjGun,X), animal(X) implies\\
\push -motive(in(pj,forest),hearBirdsChirp).

p(25) :: xmasDay implies\\
\push want(pj,foodFor(dinner)).

p(26) :: longWalk(pj) implies\\
\push -motive(in(pj,forest),practiceShooting).

p(27) :: xmasDay implies\\
\push -motive(in(pj,forest),practiceShooting).
\end{quote}
}

By the story information and parts of the background knowledge not shown above, we can derive that Papa Joe is a hunter, and that he has fired at a turkey. From the first inference, p(22) already implies that the motivation is to hunt for food. The same inference can be derived by p(25) and p(21), although for a different reason. At the same time, p(23) and p(24) dismiss the possibility of the motivation being to catch the two turkeys or to hear birds chirp, whereas story information along with either p(26) or p(27) dismiss also the possibility of the motivation being to practice shooting.

The background knowledge above follows evidence from the participant responses in our psychological study that the motives in the answers of the second question can be ``derived'' from higher-level desires or goals of the actor. Such high-level desires and intentions are examples of \emph{generalizations} that
contribute to the coherence of comprehension, and to the creation of
\emph{expectations} in readers about the course of action
that the story might follow in relation to
fulfilling desires and
achieving intentions of the protagonists.

\section{Related Work}

Automated story understanding has been an ongoing field of AI research
for the last forty years, starting with the planning and goal-oriented approaches of
Schank, Abelson, Dyer and others~\cite{SchankAbelson1977,Dyer1983}; for a
good overview see~\cite{Mueller2002} and the website~\cite{MuellerWebsite}.
Logic-related approaches have largely been concerned with the development
of appropriate representations, translations or annotations of narratives,
with the implicit or explicit assumption that standard deduction or logical
reasoning techniques can subsequently be applied to these. For example,
the work of Mueller~\cite{Mueller2003}, which in terms of story
representation is most closely related to our approach,
equates various modes of story understanding with the solving of
satisfiability problems. \cite{NiehausYoung2009} models understanding as
partial order planning, and is also of interest here because of a
methodology that includes a controlled comparison with human readers.

To our knowledge there has been very little work relating story comprehension
with computational argumentation, an exception being~\cite{BexV13},
in which a case is made for combining narrative and argumentation techniques
in the context of legal reasoning, and with which our argumentation framework shares important similarities. Argumentation for reasoning about actions
and change, on which our formal framework builds,
has been studied in \cite{VoF05,MichaelK09}.

Many other authors have emphasized the
importance of commonsense knowledge and reasoning in story comprehension~\cite{SilvaMontgomery1978,DahlgrenEtAl1989,Riloff1999,Mueller2004,Mueller2009,Verheij2009,ElsonMcKeown2009,Michael2010}, and indeed how it
can offer a basis for story comprehension tasks beyond question
answering \cite{Michael_(2013)_StoryCalculemus}.

\section{Conclusions and Future Work}

We have set up a conceptual framework for story comprehension by fusing together
knowhow from the psychology of text comprehension with established AI techniques and
theory in the areas of Reasoning about Actions and Change and Argumentation.
We have developed a proof of concept automated system to
evaluate the applicability of our framework through a similar
empirical process of evaluating human readers.
We are currently, carrying out psychological
experiments with other stories to harness
world knowledge and test our
system against the human readers.

There are still several problems that we need to
address to complete a fully automated approach to
SC, over and above the problem of extracting
through Natural Language Processing techniques
the narrative from the free format text.
Two major such problems for our
immediate future work are (a) to address further
the computational aspects of the
challenges of cognitive economy and coherence
and (b) the systematic extraction or acquisition
of common sense world knowledge. For the first
of these we will investigate how this can be addressed
by applying ``computational heuristics'' on top of
(and without the need to reexamine) the solid semantic
framework that we have developed thus far, drawing
again from psychology to formulate such heuristics.
In particular, we expect that the psychological studies will
guide us in modularly introducing computational operators
such as \emph{selection, dropping and generalization operators}
so that we can improve the coherence of the computed models.

For the problem of the systematic acquisition
of world knowledge we aim to source this
(semi)-automatically from the Web.
For this we could build on lexical databases such as
WordNet \cite{Miller1995}, FrameNet \cite{Baker1998},
and PropBank \cite{Palmer2005}, exploring the
possibility of populating the world knowledge
theories using archives for common sense knowledge
(e.g., Cyc \cite{Lenat95}) or through the automated extraction
of commonsense knowledge from text using natural language processing
\cite{MichaelEtAl_(2008)_ExperimentalKI}, and appealing to textual
entailment for the semantics of the extracted knowledge \cite{Michael_(2009)_RBTL,Michael_(2013)_MachinesWithWebsense}.

We envisage that the strong inter-disciplinary nature of our work can
provide a concrete and important test bed for evaluating the development
of NMR frameworks in AI while at the same time offering valuable
feedback for Psychology.

\bibliographystyle{aaai}
\bibliography{NMRandStoryComprehension}
\end{document}